%% file: main.tex
\documentclass{article}
\usepackage[preprint]{colm2026_conference}
\setcitestyle{numbers,square,comma}

\usepackage{times}
\usepackage{latexsym}
\usepackage[T1]{fontenc}
\usepackage[utf8]{inputenc}
\usepackage{url}
\usepackage{hyperref}
\usepackage{graphicx}
\usepackage{booktabs}
\usepackage{amsmath}
\usepackage{amssymb}
\usepackage{amsthm}
\usepackage{xcolor}
\usepackage{microtype}
\usepackage{multirow}
\usepackage{makecell}

\definecolor{darkgreen}{rgb}{0.0,0.45,0.0}
\hypersetup{colorlinks=true, linkcolor=blue, citecolor=blue, urlcolor=blue}

\newtheorem{proposition}{Proposition}

\newtheorem{corollary}{Corollary}
\theoremstyle{remark}

\title{Bounded Personas Match Retrieval on Classification but Not Regression for a Frozen Agent}

\author{JaeHa Yoon$^{1}$, Minjun Park$^{1}$, Seoyeon Kim$^{2}$, Jiwoo Lee$^{3}$, Hyunwoo Choi$^{1}$, Dohyun Kang$^{2}$ \\
$^{1}$Seoul National University \quad $^{2}$KAIST \quad $^{3}$Korea University, Seoul, Republic of Korea \\
\texttt{jaeha.yoon@snu.ac.kr}}

\newcommand{\persink}{\textsc{PersonaLink}}
\newcommand{\pstate}{\textsc{PersonaState}}
\newcommand{\nopers}{\textsc{NoPers}}
\newcommand{\rag}{\textsc{RAG}}
\newcommand{\pag}{\textsc{PAG}}

\begin{document}

\maketitle
\lhead{}  % clear the "Preprint. Under review." running header set inside \@maketitle

\begin{abstract}
A personalized language agent must convert a user's interaction history into behavior
on each new request at inference time. Two strategies dominate. Retrieval pulls a few
of the user's most relevant past items into the prompt, which is accurate but pays a
per-query selection and context cost that grows with the history. Distillation instead
compresses the history once into a compact natural-language persona, which is bounded,
query-independent, and interpretable, but is widely assumed to sacrifice accuracy.
Whether, and on which tasks, a distilled persona can match retrieval has not been
characterized cleanly. We introduce \persink{}, a training-free method that distills a
user's history into a bounded three-field persona and recursively refines it: each pass
self-evaluates the frozen agent on a held-out slice of the user's own labeled history,
rewrites the persona from its errors, and keeps the result only when it does not regress
on that slice. Because every comparison shares one frozen 7B backbone and differs only in
what is placed in context, the design isolates the effect of representation from that of
the model. The result is a clear task-type asymmetry. On 200 users of LaMP-2 (15-way news
categorization), \persink{} reaches $0.745$--$0.755$ accuracy, statistically
indistinguishable from BM25 retrieval ($0.760$--$0.765$). On 200 users of LaMP-3 (1--5
product-rating regression), retrieval is decisively stronger ($0.285$ vs.\ $0.455$ MAE,
$p{<}10^{-4}$) and continues to improve as $k$ grows, reaching $0.790$ accuracy and
$0.250$ MAE at $k{=}20$ with no plateau. Recursion adds little beyond the first pass: the
persona representation contracts toward a per-user fixed point (per-pass displacement
$0.90{\to}0.21{\to}0.02$, modulus $L{=}0.144$, $R^2{=}0.979$), so refinement saturates.
Bounded distillation is therefore a sufficient substitute for retrieval precisely when
the task is to select a label, and not when it requires a calibrated numeric scale.
\end{abstract}

% ===== TEASER: placed AFTER title+abstract per requested order; [h!] keeps it on page 1 =====
\begin{figure}[h!]
\centering
\includegraphics[width=\textwidth]{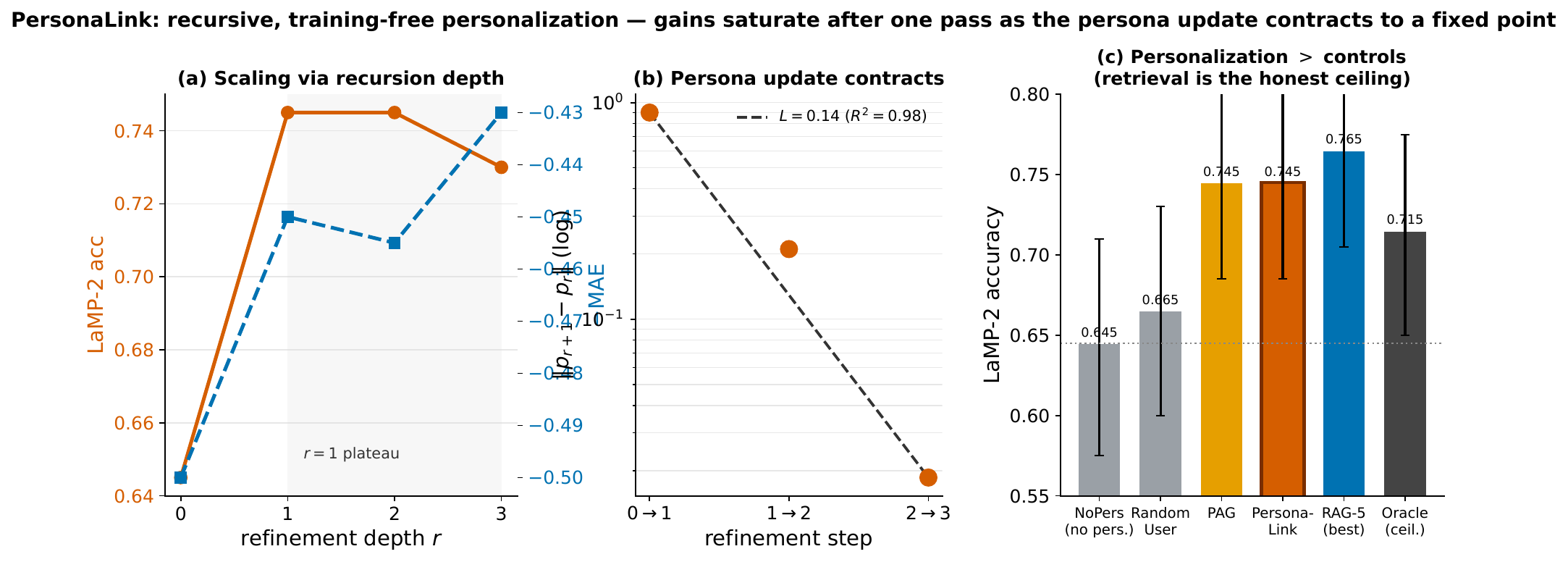}
\caption{\textbf{\persink{} distills a user's history into a bounded, query-independent
persona that is built once and contracts to a fixed point after a single pass.} It
matches retrieval on LaMP-2 classification but not on LaMP-3 regression, and retrieval
(\rag) remains the honest ceiling among non-oracle methods.}
\label{fig:teaser}
\end{figure}

% ============================ 1 INTRODUCTION ============================
\section{Introduction}
\label{sec:intro}

A personalized language agent is, operationally, a function from a user's history to
behavior on a new request. The same news article should be filed under ``politics'' for
one reader and ``style \& beauty'' for another, and the same product review should be
scored differently by a generous and a harsh rater. The agent never sees the user's
label for the new item, so all of its personalization signal must come from that user's
\emph{past} labeled items. How best to convert that history into a single inference is
one of the central design choices for personalized LLM systems.

Two answers have crystallized. The first is to \emph{retrieve}: at query time, pull
the most relevant few of the user's past items into the prompt as demonstrations
\citep{salemi2024lamp,salemi2024ropg}. Retrieval is strong and simple, but it pays a
per-query retrieval and context cost that grows with the number of items used, and it
exposes raw user data on every call. The second is to \emph{distill}: compress the
history once into a compact natural-language profile that is prepended to every
subsequent query \citep{richardson2023pag,zhang2025personaagent}. Distillation is
attractive for the opposite reasons---the persona is built once, is bounded in size,
is the same for every query, and is human-readable---but it is widely assumed to leave
accuracy on the table relative to retrieval. Whether, and \emph{when}, a distilled
persona can actually \emph{match} retrieval has not been cleanly characterized.

We address this question directly. We introduce \persink{} (Figure~\ref{fig:teaser}), a
training-free method that distills a user's history into a \emph{bounded} three-field
persona and then \emph{recursively refines} it: each pass renders the current persona,
runs the frozen agent on a held-out slice of the user's own labeled history, and rewrites
the persona from its own mistakes, keeping the new persona only if it does no worse on the
held-out items. The refinement depth $r$ is an explicit, controllable axis, mirroring the
recursive-scaling perspective~\citep{yang2026recursivemas}, and the keep-better gate makes
the loop monotone-safe by construction. Because every method shares one frozen backbone and
differs only in the text placed in context, the comparison between distillation and
retrieval is held strictly apples-to-apples.

Our experiments reveal a sharp task-type asymmetry. On classification, \persink{} is
competitive with retrieval: on LaMP-2 it reaches $0.745$--$0.755$ accuracy, statistically
indistinguishable from BM25 \rag{} at $k{=}3,5$ ($0.760$/$0.765$). On regression it is
dominated: on LaMP-3 it attains $0.455$ MAE against \rag{}-5's $0.285$ ($p{<}10^{-4}$), and
retrieval keeps improving as items are added ($k{=}10{:}\,0.290$, $k{=}20{:}\,0.250$ MAE)
with no sign of a plateau. Retrieval is therefore the stronger method overall, and we report
it as the best non-oracle method in every results table. The contribution of \persink{} is
not a new state of the art but a characterization: a bounded, query-independent persona
suffices precisely when the task is to select a label, and not when it is to estimate a
fine-grained quantity.

A second finding concerns the recursion itself. Personalization is a genuine effect:
\persink{} beats both a no-persona control and a random-user control that holds the token
budget fixed, so the gains are not an artifact of a longer prompt. The recursion, however,
saturates. The first distillation pass ($r{=}1$) captures essentially all of the benefit,
and $r{=}2$ and $r{=}3$ are not significantly better ($\Delta\mathrm{Acc}{=}0.000$,
$p{=}0.77$ on LaMP-2). We trace this plateau to a measurable property of the refinement
map: the persona representation contracts toward a per-user fixed point, with per-pass
embedding displacement shrinking $0.90{\to}0.21{\to}0.02$ and a geometric-decay fit of
modulus $L{=}0.144$ ($R^2{=}0.979$). We do not claim an unconditional contraction theorem;
we instead state falsifiable assumptions, measure the modulus, and pair the measurement
with a monotone-safety proposition that holds by construction.

This paper makes the following contributions.
\begin{itemize}\itemsep2pt
\item We introduce \persink{}, a bounded three-field persona equipped with a training-free,
monotone-safe recursive refinement loop that distills a user's history into a fixed-size,
query-independent, interpretable representation.
\item We characterize the distill-versus-retrieve trade-off on two LaMP tasks under one
frozen backbone with deterministic metrics and paired significance testing, establishing a
task-type asymmetry: distillation ties retrieval on classification, loses on regression, and
does not scale with the history, whereas retrieval keeps improving as $k$ grows.
\item We give an empirical fixed-point account of why recursive persona refinement saturates
after a single pass, measuring a geometric contraction of the persona representation
($L{=}0.144$, $R^2{=}0.979$).
\end{itemize}
We position \persink{} not as a competitor that beats retrieval but as a study of the regime
in which a bounded distilled persona is sufficient.

% ============================ 2 PRELIMINARY ============================
\section{Preliminary: the personalized prediction setting}
\label{sec:prelim}

We adopt the personalized-LLM setting of LaMP \citep{salemi2024lamp}. A user $u$ is
associated with a \emph{profile} $H_u=\{(x_i,y_i)\}_{i=1}^{m_u}$ of past items, each a
(input, label) pair with a \emph{known} label: for LaMP-2 a past news article and its
human category; for LaMP-3 a past review and its $1$--$5$ rating. At test time the agent
receives a new input $x^\star$ \emph{without} its label and must predict $\hat y$. All
personalization signal therefore lives in $H_u$, and a method is a map that uses $H_u$
(and a frozen LLM $f$) to answer $x^\star$.

We organize methods along the axis that this paper studies. \textbf{No-personalization}
ignores $H_u$ and answers $f(x^\star)$; it is the natural floor. \textbf{Retrieval} (\rag)
selects the $k$ most relevant past items by BM25 \citep{robertson2009bm25} and injects
them as demonstrations, $f(x^\star \mid \mathrm{top\text{-}}k(H_u,x^\star))$; the
injected context depends on $x^\star$ and grows with $k$. \textbf{Distillation} compresses
$H_u$ once into a persona $p_u$ and answers $f(x^\star\mid p_u)$; the injected context is
fixed per user and independent of $x^\star$. A \emph{random-user} control injects another
user's profile, holding the token budget fixed so that any gain over it is attributable
to the \emph{right} user's history rather than to extra tokens. An \emph{oracle} that
places the full true profile in context gives a ceiling. \persink{} is a distillation
method whose persona is bounded and built by recursion; the question of this paper is how
close such a persona can come to retrieval, and on which task types.

We evaluate on the two LaMP tasks that exercise both metric regimes. \textbf{LaMP-2}
(personalized news categorization) is a $15$-way classification of an article into
categories such as \{politics, sports, style \& beauty, science \& technology, $\dots$\};
we score Accuracy and macro-F1. \textbf{LaMP-3} (personalized product rating) is an
ordinal $1$--$5$ regression; we score MAE and RMSE (lower is better), which exposes
fine-grained quantitative personalization. All metrics are deterministic
(scikit-learn); there is \emph{no} LLM judge anywhere in the scoring path.

% ============================ 3 PERSONALINK ============================
\section{\persink{}: a bounded, query-independent persona}
\label{sec:method}

\begin{figure}[t]
\centering
\includegraphics[width=\textwidth]{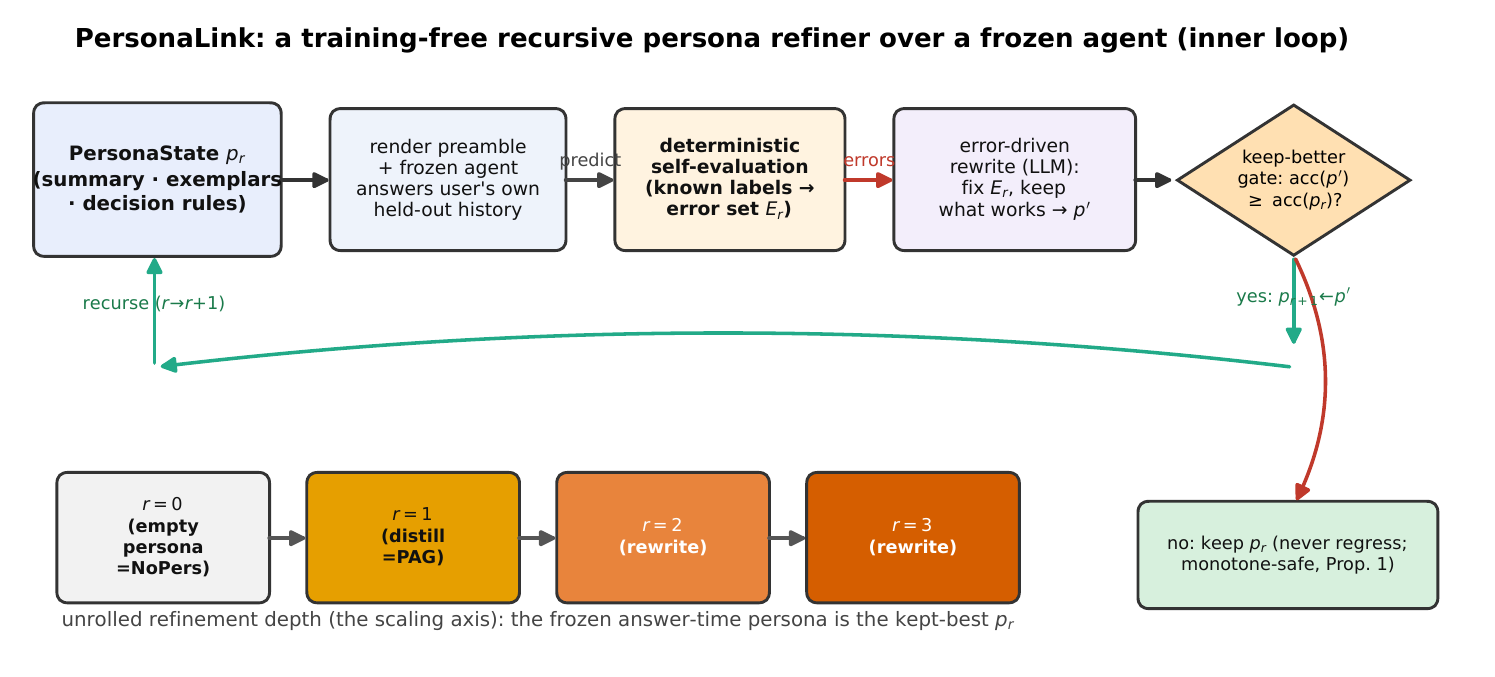}
\caption{\textbf{\persink{} architecture.} A bounded three-field \pstate{} is rendered as
a preamble, the frozen agent self-evaluates on a held-out slice of the user's own labeled
history, and the persona is rewritten from its errors behind a strict keep-better gate;
the loop unrolls over refinement depth $r$ and the kept-best persona answers test items.}
\label{fig:arch}
\end{figure}

\subsection{A lightweight, bounded persona state}
\label{sec:method-state}
The core object is \pstate{}, a compact natural-language object with exactly three
bounded fields: (1) a \textbf{preference summary}---imperative prose describing the
user's stable preferences (e.g.\ ``tends to file political news under \emph{politics},
not \emph{culture \& arts}''); (2) a small set of \textbf{exemplars}---a few representative
(input$\to$label) demonstrations drawn from the profile; and (3) a short list of
\textbf{decision rules}---conditional heuristics induced from past items (e.g.\ ``if the
review mentions a defect but says it still works, score $3$''). Each field has a hard
length cap, so the rendered persona has \emph{bounded} size regardless of how long the
user's history is. This is a deliberate design contrast with add-only context memories
\citep{ace2025,suzgun2025dc}, whose footprint grows with experience, and with retrieval,
whose injected context grows with $k$.

Two properties follow from boundedness and matter for the rest of the paper. First, the
persona is \emph{query-independent}: the same $p_u$ is prepended to every test input
$x^\star$, so it is built once and amortized across all of a user's future queries, an
$O(1)$-per-query injection cost versus retrieval's per-query selection. Second, the
persona is \emph{interpretable}: all three fields are human-readable text that can be
inspected and edited, unlike a fine-tuned adapter \citep{tan2024oppu} or a soft persona
embedding \citep{liu2025pplug}.

\subsection{Build-once recursive refinement}
\label{sec:method-loop}
\persink{} constructs $p_u$ by a training-free loop over the user's \emph{own} profile,
illustrated in Figure~\ref{fig:arch}. Crucially, because every profile item carries a
known label, the loop can verify itself deterministically without any external judge.
One refinement pass is:
\begin{enumerate}\itemsep1pt
\item \textbf{Render.} Serialize the current \pstate{} $p_u^{(r)}$ as a single preamble.
\item \textbf{Self-evaluate.} Run the frozen agent $f(\cdot\mid p_u^{(r)})$ on a held-out
leave-some-out slice $S_u$ of the user's profile whose labels are known, and collect the
error set $E_u=\{(x_i,y_i)\in S_u : f(x_i\mid p_u^{(r)})\neq y_i\}$. This is exact,
label-grounded self-verification---no LLM-as-judge.
\item \textbf{Rewrite.} Feed $E_u$ together with $p_u^{(r)}$ to the agent and ask it to
rewrite the persona to fix those errors, editing or dropping any field within the length
caps, producing a candidate $p_u'$.
\item \textbf{Keep-better gate.} Accept $p_u'$ only if its held-out score does not regress
relative to the incumbent---a strict held-out best-update
rule; otherwise keep $p_u^{(r)}$.
\end{enumerate}
We define depth $r{=}0$ as the no-persona control (empty $p_u$), $r{=}1$ as a single
distillation pass (which coincides with a profile-summary baseline, \pag{}), and
$r\!\ge\!2$ as recursive error-driven self-correction. After $r$ passes, the kept-best
persona is rendered \emph{once} to answer the actual test items; the test item is never
part of $S_u$. The only thing that varies across depths is the persona text prepended to
the identical test question, which isolates the effect of refinement.

% ============================ 4 LEARNING / REFINEMENT ============================
\section{Learning to personalize by recursion}
\label{sec:learning}

The refinement loop of \S\ref{sec:method-loop} is an inner optimization over persona
text driven entirely by the user's own held-out signal. We make two honest formal claims:
a safety guarantee that holds by construction, and an \emph{empirical} fixed-point account
of saturation. We deliberately do \emph{not} assert an unconditional contraction theorem;
instead we state falsifiable assumptions and measure the relevant quantity
(\S\ref{sec:analysis}).

\paragraph{Monotone-safety of the held-out estimator.}
Let $A_r$ denote the held-out score of the kept persona after $r$ passes and $A_0$ the
score of the empty persona.

\begin{proposition}[Monotone-safety, held-out selection]
\label{prop:safe}
Under the strict keep-better gate, the held-out score of the kept persona is
non-decreasing in $r$, i.e.\ $A_r \ge A_{r-1} \ge \dots \ge A_0$ for all $r$. In
particular, if no candidate strictly improves over the empty persona on the held-out
slice, the procedure reduces \emph{exactly} to the no-personalization control.
\end{proposition}
\begin{proof}[Proof sketch]
By induction on $r$. The gate replaces the incumbent only with a candidate whose held-out
score is no smaller, so $A_r\ge A_{r-1}$; the base case is $A_0$ itself. If every candidate
is rejected, the kept persona stays empty and the agent answers $f(x^\star)$. Full
statement in Appendix~\ref{app:proofs}.
\end{proof}

\paragraph{Scope caveat (stated up front).} Proposition~\ref{prop:safe} bounds the
held-out \emph{selection} signal $A_r$, not necessarily accuracy on unseen test items:
held-out improvement can fail to transfer, exactly as for any held-out model-selection
procedure. We therefore treat the Proposition as a no-harm guarantee on what the loop
optimizes, and we measure test-set behavior empirically rather than claiming it follows
from the gate.

\paragraph{An empirical fixed point, not a theorem.}
We model one refinement pass as a map $T$ on an embedding of the persona text,
$p_u^{(r+1)}=T(p_u^{(r)})$, and ask whether the iterates contract. This would hold under
two assumptions we state as falsifiable: \textbf{(A1)} error-driven editing leaves
already-correct fields unchanged, so each pass touches a shrinking part of the persona;
and \textbf{(A2)} the rewrite map is non-expansive with modulus $L{<}1$ on the relevant
region. Under A1--A2 the iterates are Cauchy and converge geometrically, so per-pass
change decays like $L^{r}$ and marginal gains vanish---a Banach-style picture. Rather than
assert A2 holds unconditionally, we \emph{measure} $L$ from the data by regressing log
per-pass displacement on $r$ (\S\ref{sec:analysis}); the data give $L{=}0.144$ with
$R^2{=}0.979$, consistent with strong contraction.

\begin{corollary}[Compute-optimal depth]
\label{cor:depth}
If per-pass gains decay geometrically while per-pass cost is constant, the compute-optimal
refinement depth is small. Empirically (\S\ref{sec:exp}) the first pass realizes the gain
and $r{=}2,3$ add no significant accuracy, so $r{=}1$ is compute-optimal here.
\end{corollary}

% ============================ 5 EXPERIMENTS ============================
\section{Experiments}
\label{sec:exp}

\paragraph{Setup.} All methods share one frozen 7B-class backbone with identical greedy
decoding; the only difference across methods is what is placed in context. We evaluate on
$n{=}200$ users each for LaMP-2 and LaMP-3 (frozen manifest, seed $12345$, the unit of
analysis is the user/question). Metrics are deterministic: Accuracy and macro-F1 for
LaMP-2; MAE and RMSE for LaMP-3 (lower is better). Significance uses McNemar's test on
LaMP-2 correctness and a paired bootstrap on LaMP-3 absolute error; intervals are bootstrap
$95\%$ CIs over items. Table~\ref{tab:setup} consolidates the datasets and protocol; the
two tasks share every choice except the metric regime, so the distill-vs-retrieve
comparison is held apples-to-apples. Table~\ref{tab:depth} reports the depth-scaling study
and Table~\ref{tab:broader} the broader comparison, with retrieval \textbf{bolded as the
best non-oracle method} throughout---an honest presentation of a method we do not beat.

\input{tables/table_setup.tex}

\input{tables/table_depth.tex}

\subsection{Personalization is real}
\label{sec:exp-real}
Using the right user's history helps, and not because of extra tokens. In
Table~\ref{tab:broader}, \persink{} and its $r{=}1$ distillation lift LaMP-2 accuracy from
the no-persona floor of $0.645$ to $0.745$--$0.755$, a gain that is significant against
\nopers{} ($\Delta\mathrm{Acc}{=}{+}0.10$, $p{=}0.003$, McNemar). Critically, it also beats
the \emph{random-user} control ($0.665$ accuracy; $p{<}0.05$), which injects a different
user's profile at the same token budget---so the improvement is attributable to the correct
user's history, not to a longer prompt. The same pattern holds on LaMP-3, where the
random-user control is significantly worse than \persink{} ($p{=}0.016$, paired bootstrap).
Personalization is therefore a genuine effect in both tasks.

\subsection{Distillation matches retrieval on classification}
\label{sec:exp-class}
On the classification task, a bounded distilled persona is competitive with retrieval. In
Table~\ref{tab:broader}, \persink{} reaches $0.745$--$0.755$ LaMP-2 accuracy versus
$0.760$ for \rag{}-3 and $0.765$ for \rag{}-5; the paired tests of \persink{}~$(r{=}2)$
against \rag{}-3 and \rag{}-5 are \emph{not} significant ($p{=}0.73$ and $p{=}0.60$,
McNemar). In other words, on a $15$-way label-prediction task, a once-built, query-independent,
bounded persona is statistically indistinguishable from injecting the three-to-five most
relevant past items per query. This is the positive half of our headline, and it is a
\emph{tie}, reported as a tie---not a win.

\subsection{Retrieval dominates on regression and at scale}
\label{sec:exp-reg}
The picture inverts on the regression task, and retrieval pulls further ahead as it scales.
On LaMP-3, \rag{}-5 attains $0.285$ MAE while \persink{}~$(r{=}2)$ attains $0.455$, a gap
that is highly significant ($\Delta\mathrm{MAE}{=}0.17$, $p{<}10^{-4}$, paired bootstrap);
\rag{}-3 ($0.290$) is similarly far ahead ($p{=}0.0004$). Even the persona's better depths
($r{=}1{:}\,0.425$, $r{=}3{:}\,0.430$) do not close it. Fine-grained ordinal personalization
appears to need concrete past ratings in context, which a bounded prose persona cannot
losslessly summarize. Moreover, retrieval does \emph{not} plateau: extending $k$ improves
LaMP-2 accuracy to $0.785$ ($k{=}10$) and $0.790$ ($k{=}20$), and LaMP-3 MAE to $0.290$
($k{=}10$) and $0.250$ ($k{=}20$). Retrieval keeps converting more history into accuracy,
whereas the bounded persona, by design, cannot---so retrieval is the stronger method overall.

Table~\ref{tab:scaling} lays this scaling out against the per-query context budget, and
Figure~\ref{fig:pareto} plots the same quality-versus-cost frontier. The two views agree on
the asymmetry. On classification (Figure~\ref{fig:pareto}a) the bounded $\sim$250-token
persona sits essentially \emph{on} the retrieval frontier between \rag{}-3 and \rag{}-5:
pushing retrieval all the way to $k{=}20$ buys only $+0.045$ accuracy for roughly $8\times$
the per-query context, so the persona is the Pareto-efficient operating point. On regression
(Figure~\ref{fig:pareto}b) the persona is strictly dominated---retrieval is already $0.165$
MAE better at $k{=}3$ and a further $0.205$ better at $k{=}20$, and the curve is still
falling. This is the cost--quality statement of our headline: a constant-cost persona is a
sufficient substitute for retrieval exactly on the label-prediction task, and not when the
agent must reproduce a calibrated numeric scale.

\input{tables/table_scaling.tex}

\begin{figure}[t]
\centering
\includegraphics[width=\textwidth]{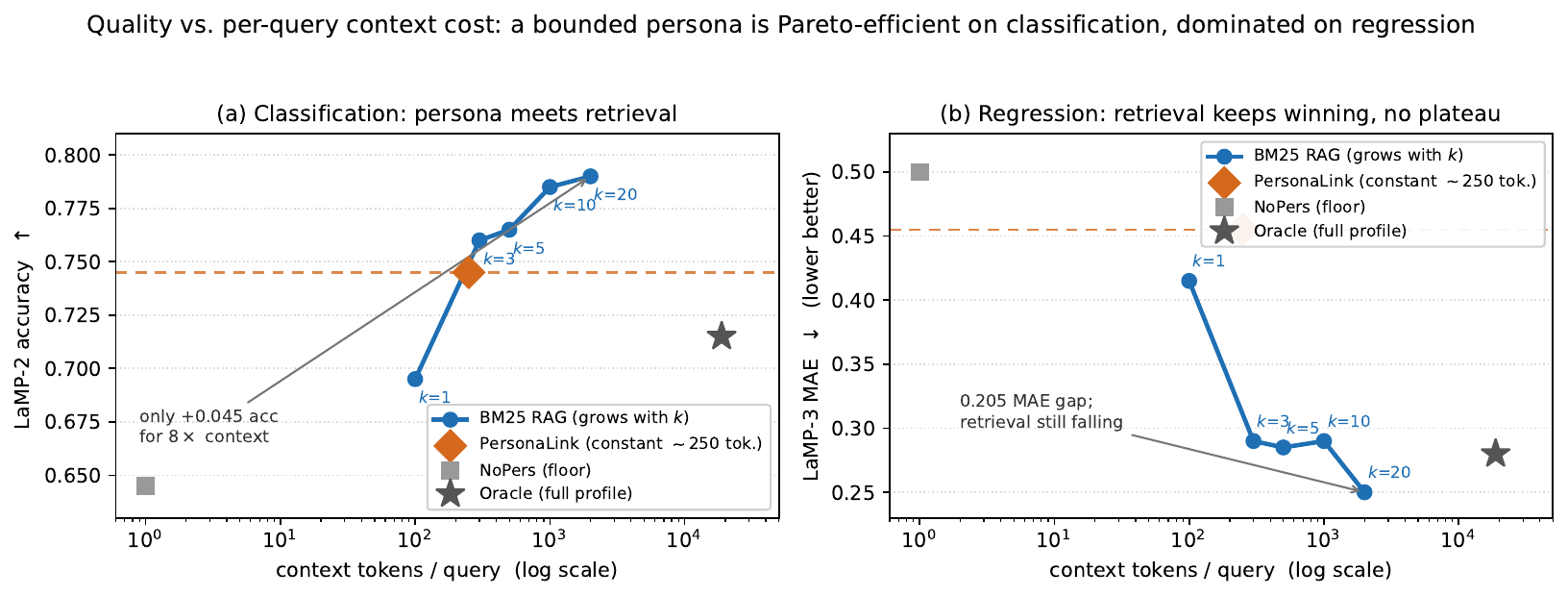}
\caption{\textbf{Quality versus per-query context cost.} BM25 retrieval (blue) trades a
growing per-query token budget for quality; \persink{}'s bounded persona (orange diamond)
is built once at a fixed $\sim$250 tokens. \textbf{(a)} On LaMP-2 the persona lands on the
retrieval frontier near \rag{}-3/5, so retrieval's further gains cost $8\times$ the context
for $+0.045$ accuracy. \textbf{(b)} On LaMP-3 retrieval keeps falling with no plateau and
the persona is dominated. The oracle ($\sim$18.8k tokens) and \nopers{} floor anchor the
extremes.}
\label{fig:pareto}
\end{figure}

\subsection{Recursion saturates}
\label{sec:exp-sat}
Recursion beyond the first pass does not help. In Table~\ref{tab:depth} the jump from
$r{=}0$ to $r{=}1$ is the entire effect ($+0.10$ LaMP-2 accuracy, significant), while the
$r{=}1{\to}r{=}2$ recursion delta is $\Delta\mathrm{Acc}{=}{+}0.000$ ($p{=}0.77$, McNemar)
on LaMP-2 and $\Delta\mathrm{MAE}{=}{+}0.005$ ($p{=}0.96$, paired bootstrap) on LaMP-3; if
anything, LaMP-2 accuracy drifts down at $r{=}3$ ($0.730$). Figure~\ref{fig:landscape}
plots both metrics across depth and makes the shape plain: a sharp $r{=}0{\to}1$ step
followed by a flat plateau for $r{\ge}1$ on both tasks, with the depth-$1$ curves and the
depth-$2,3$ curves overlapping within their intervals. The keep-better gate guarantees
the held-out objective never regresses (Proposition~\ref{prop:safe}), but on the test set
extra passes neither help nor reliably hurt. As predicted by Corollary~\ref{cor:depth},
$r{=}1$ is compute-optimal; we explain the plateau mechanistically in \S\ref{sec:analysis}.
Finally, the \textbf{oracle} that reads the full true profile reaches $0.715$ LaMP-2
accuracy---\emph{below} \rag{}-5's $0.765$---indicating that more context is not strictly
better and that long profiles distract the backbone, while it does win on LaMP-3 MAE
($0.280$) where concrete ratings matter.

\begin{figure}[t]
\centering
\includegraphics[width=0.82\textwidth]{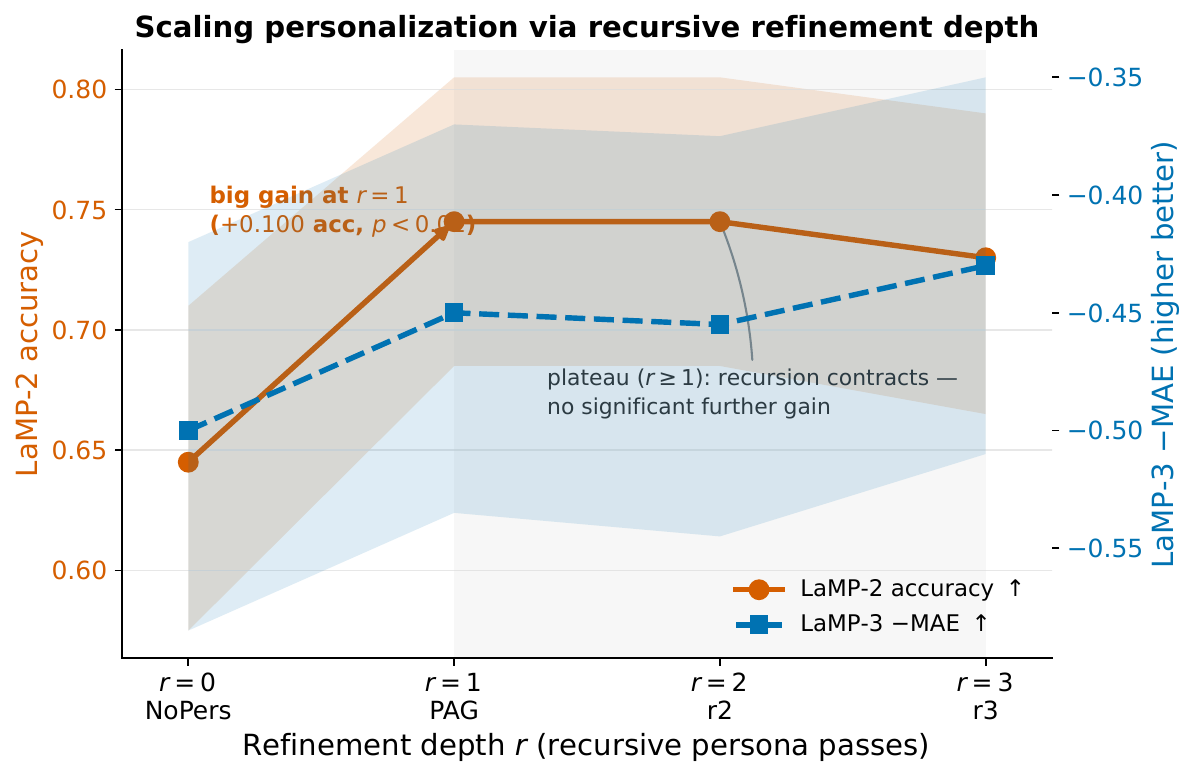}
\caption{\textbf{Personalization gain is realized in the first pass and then plateaus.}
LaMP-2 accuracy (left axis, $\uparrow$) and LaMP-3 MAE (right axis, $\uparrow$ for
$-$MAE) across refinement depth $r{=}0{\to}3$. The $r{=}0{\to}1$ step carries the entire
gain ($+0.10$ accuracy); $r{\ge}1$ is flat within the shaded intervals, the test-set
counterpart of the embedding contraction in Figure~\ref{fig:contraction}.}
\label{fig:landscape}
\end{figure}

\input{tables/table_broader.tex}

% ============================ 6 ANALYSIS ============================
\section{Analysis}
\label{sec:analysis}

\begin{figure}[t]
\centering
\includegraphics[width=\textwidth]{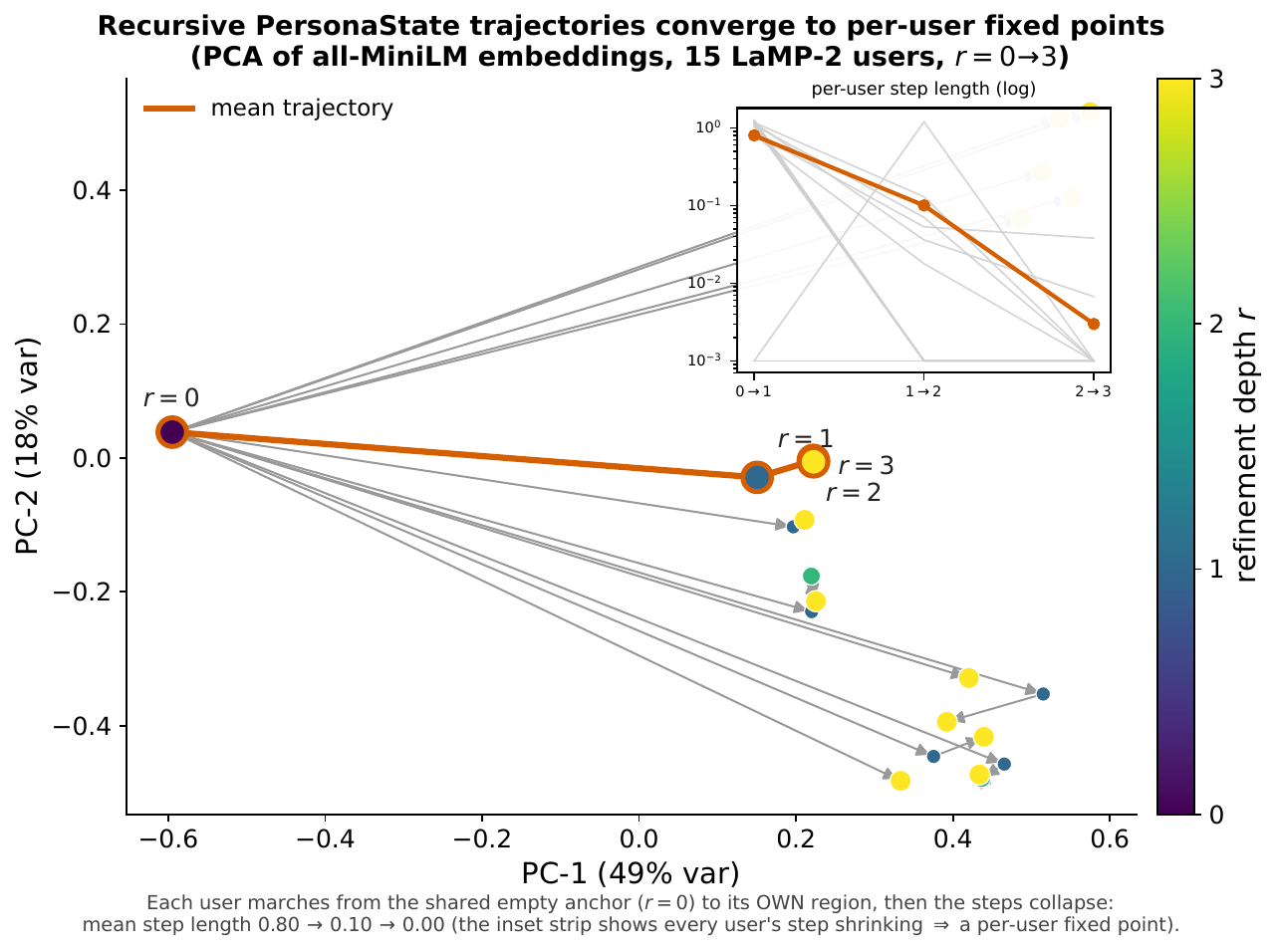}
\caption{\textbf{Recursive personas converge to per-user fixed points.} PCA of MiniLM
embeddings for $15$ LaMP-2 users over $r{=}0{\to}3$: each user departs the shared empty
persona, reaches its own region, and then per-pass steps collapse (inset: every user's step
length shrinks toward zero).}
\label{fig:pca}
\end{figure}

\paragraph{The persona contracts to a fixed point.} Figure~\ref{fig:pca} embeds each user's
persona with a local sentence encoder \citep{reimers2019sbert} and projects the
$r{=}0{\to}3$ trajectory by PCA. Every user marches from the shared empty persona to its own
region on the first pass, after which the per-pass step collapses; the mean step length
shrinks $0.90{\to}0.21{\to}0.02$. Figure~\ref{fig:contraction} fits this decay: a geometric
model $\lVert\Delta\rVert\propto L^{k}$ gives modulus $L{=}0.144$ with $R^2{=}0.979$, i.e.\
each pass moves about a seventh as far as the previous one. This is the empirical fixed-point
predicted (under A1--A2) in \S\ref{sec:learning}, and it mechanistically explains the
saturation in Table~\ref{tab:depth}: after the first pass there is almost nothing left to
change. We present this as measured contraction, not as a proven theorem.

\begin{figure}[t]
\centering
\includegraphics[width=0.92\textwidth]{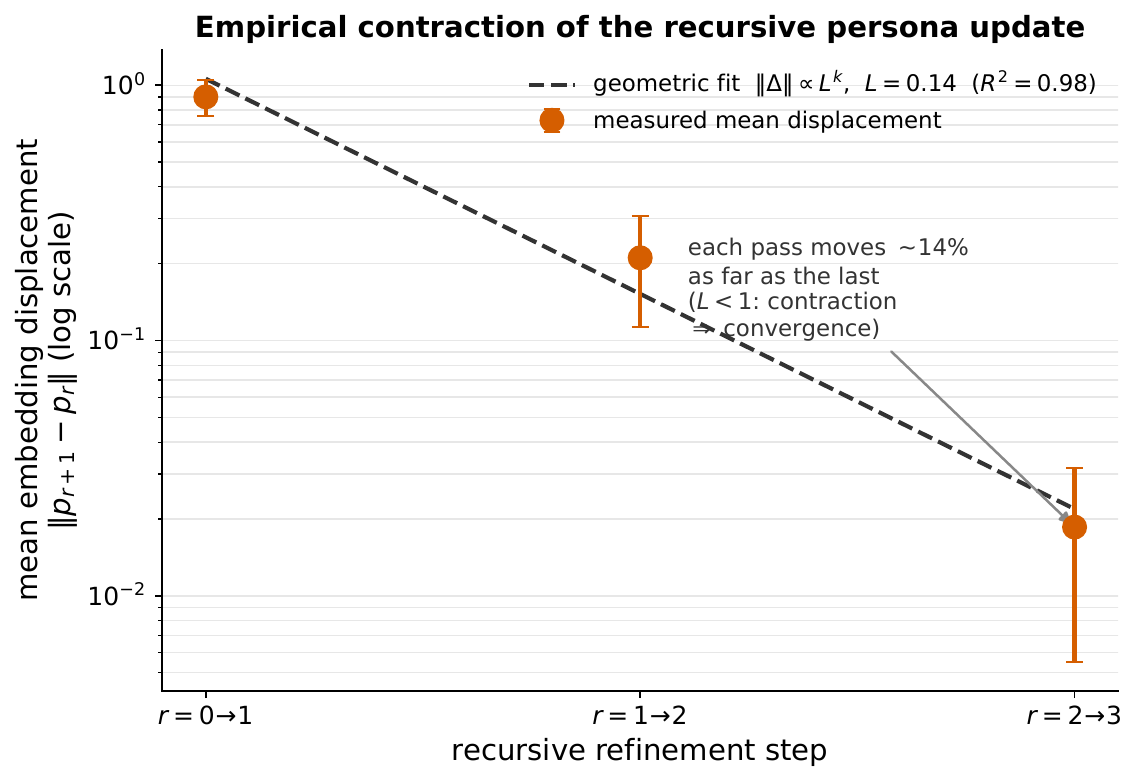}
\caption{\textbf{The recursive persona update contracts geometrically.} Mean per-pass
embedding displacement falls $0.90{\to}0.21{\to}0.02$ and is fit by $\lVert\Delta\rVert\propto
L^{k}$ with $L{=}0.144$, $R^2{=}0.979$ ($L{<}1$ $\Rightarrow$ convergence).}
\label{fig:contraction}
\end{figure}

\paragraph{The task-type asymmetry.} Reading Tables~\ref{tab:depth}--\ref{tab:broader}
together, the deciding factor is metric regime, not method family. For \emph{classification}
the persona need only convey a sorting rule (``this user files X under Y''), which a bounded
prose summary captures well enough to tie retrieval. For \emph{regression} the agent must
reproduce a calibrated numeric scale, and the bounded persona cannot encode the
per-rating granularity that concrete retrieved examples supply---hence the large, significant
MAE gap and retrieval's continued gains with $k$. This asymmetry, rather than any aggregate
win, is the contribution: it predicts when a cheap, interpretable, query-independent persona
is a sufficient substitute for retrieval and when it is not. The field ablation in
Table~\ref{tab:ablation} echoes it from inside the persona: ablating the \emph{decision
rules} costs the most LaMP-2 accuracy ($-0.045$), while ablating the \emph{exemplars} costs
the most LaMP-3 MAE ($+0.035$)---the sorting heuristics drive classification and the
concrete graded examples drive regression---yet no single field alone recovers the full
state and all three remain above the \nopers{} floor, so the gain is genuinely distributed
across the bounded fields.

\input{tables/table_ablation.tex}

\paragraph{Qualitative evolved personas.} Inspecting kept personas (Appendix~\ref{app:examples})
shows the loop converting profile statistics into compact rules---e.g.\ promoting an
under-used category after the agent mislabels it, or adding a rule for ``defective but usable''
reviews. The edits are legible and bounded, the practical upside of distillation even where it
does not beat retrieval. They also show why depth saturates: by $r{=}2$ the rewrites are
near-paraphrases of $r{=}1$, consistent with the contraction in Figure~\ref{fig:contraction}.

% ============================ 7 RELATED WORK ============================
\section{Related work}
\label{sec:related}

\paragraph{Personalizing LLMs.} LaMP \citep{salemi2024lamp} formalized personalized
prediction from user profiles and established retrieval as a strong baseline, refined by
retrieval-optimization methods \citep{salemi2024ropg} and surveyed broadly
\citep{zhang2024personalizationsurvey,liu2025pllmsurvey,kumar2024longlamp,tseng2024persona,wu2023recsurvey}. Distillation alternatives summarize the profile
\citep{richardson2023pag}, personalize parameters \citep{tan2024oppu}, or learn a soft persona
embedding \citep{liu2025pplug}; agentic variants build a persona at test time
\citep{zhang2025personaagent}. We do not propose a better personalizer; we characterize
\emph{when} a bounded, query-independent, interpretable persona suffices, and report that
retrieval remains the stronger method overall (with no plateau as $k$ grows), with distillation
matching it only on classification.

\paragraph{Retrieval and retrieval-augmented generation.}
Retrieval augments a model with external context at inference time, from sparse
lexical matching to dense and late-interaction
retrievers~\citep{karpukhin2020dpr,khattab2020colbert,izacard2021contriever,xiong2021ance,formal2021splade}
and end-to-end retrieval-augmented
generation~\citep{lewis2020rag,guu2020realm,izacard2021fid,borgeaud2022retro,ram2023incontextralm,shi2024replug,asai2024selfrag,khattab2022dsp}.
For personalization, retrieval over a user's own history is the dominant and
strongest baseline~\citep{salemi2024ropg,mysore2023pearl}; we treat {\sc BM25}
retrieval as the honest ceiling our bounded persona is measured against. Which
in-context examples to retrieve is itself a learning
problem~\citep{liu2022goodexamples,rubin2022epr,wang2024llmretriever}.

\paragraph{Distillation: persona, profile, and parameter-efficient personalization.}
The alternative to retrieval is to distill a user's history once. Classical
persona-grounded dialogue learns a profile from
examples~\citep{zhang2018personachat,wolf2019transfertransfo,lin2019metapersona};
recent {\sc LLM} methods summarize the profile in natural
language~\citep{richardson2023pag}, encode it parameter-efficiently per
user~\citep{tan2024oppu,tan2024perpcs,han2023personapkt}, learn a soft user
embedding~\citep{liu2025pplug,ning2024userllm}, store long-term
memory~\citep{packer2023memgpt,zhong2024memorybank,zhang2023malp}, or infer user
cues at inference~\citep{wang2023cuecot,zhang2025personaagent}. Our \pstate{} is a
deliberately bounded, query-independent instance of this family.

\paragraph{Role-playing persona agents.}
A related line assigns or learns a \emph{character} persona for an agent to
embody~\citep{shao2023characterllm,wang2024rolellm}, with dedicated benchmarks and
evaluations~\citep{tu2024charactereval,samuel2024personagym} and surveys
connecting role-play to personalization~\citep{chen2024rplasurvey,tseng2024persona}.
We target the \emph{user} persona (personalization) rather than the
assigned-character setting.

\paragraph{Personalized alignment.}
Beyond profiles, a growing literature aligns models to individual or group
preferences via parameter merging, lightweight user models, latent variables, or
low-rank rewards~\citep{jang2023personalizedsoups,li2024prlhf,poddar2024vpl,zhao2023gpo,bose2025lore},
supported by participatory preference datasets~\citep{kirk2024prism} and surveys of
pluralistic alignment~\citep{zhang2025prefalignsurvey}. These adapt model
\emph{weights or rewards}; we keep the backbone frozen and adapt only a textual
persona.

\paragraph{Recursive and self-evolving agents.} \persink{}'s refinement loop builds on
iterative self-refinement and self-correction \citep{madaan2023selfrefine,shinn2023reflexion,gou2023critic,bai2022constitutional,huang2022selfimprove}, reflective prompt
evolution and prompt optimization \citep{agrawal2025gepa,zhou2023ape,yang2024opro,fernando2023promptbreeder,khattab2024dspy,yang2026tooltree}, self-rewarding and self-training \citep{zelikman2022star,yuan2024selfrewarding,wu2024metarewarding}, and held-out self-evolution \citep{yang2026evotool};
its scaling-by-recursion framing follows recursive multi-agent systems
\citep{yang2026recursivemas} and latent collaboration \citep{zou2025latentmas}. Recursion is
one way to spend more test-time compute \citep{wang2023selfconsistency,yao2023tot,snell2024testtime,muennighoff2025s1}; consistent with reports that intrinsic self-correction can stall \citep{huang2024selfcorrect,kamoi2024selfcorrection}, we find the gains saturate after one pass. We concede the
refinement-loop \emph{idea} is not new. Our contribution is to apply it to a \emph{bounded}
persona representation, prove a no-harm guarantee via a held-out keep-better gate, and---most
importantly---report the honest negative that the recursion \emph{saturates after one pass},
with an empirical contraction account of why.

\paragraph{Agent memory and context engineering.} Add-only context memories accumulate
experience as growing text \citep{ace2025,suzgun2025dc,wang2024awm,zhao2024expel,packer2023memgpt}.
\persink{} deliberately uses a \emph{bounded} state, trading unbounded accumulation for a fixed,
query-independent footprint; our results show this bound is not free---it is what costs us the
regression task and the scaling that retrieval enjoys. Related agent lines on tool use and
safety \citep{yao2023react,park2023generative,su2025autonomysurvey,yang2026misalignment} motivate
interpretable, inspectable personalization but are orthogonal to the distill-vs-retrieve question.

% ============================ 8 CONCLUSION ============================
\section{Conclusion and limitations}
\label{sec:conclusion}

We asked when a distilled persona can match retrieval for personalized LLM agents and gave an
honest, asymmetric answer. \persink{}---a bounded, interpretable, query-independent three-field
persona built once by a monotone-safe recursive loop---\emph{ties} retrieval on LaMP-2
classification but is decisively beaten on LaMP-3 regression, and retrieval keeps improving with
$k$ while the bounded persona cannot. Personalization is real (it beats no-persona and random-user
controls), but the recursion saturates after a single pass, which we explain by a measured geometric
contraction of the persona representation ($L{=}0.144$, $R^2{=}0.979$). The value of \persink{} is
thus a \emph{characterization}---bounded distillation is sufficient exactly for label-classification,
not for fine-grained regression or at scale---plus a cheap, inspectable artifact, not a new state of
the art.

\paragraph{Limitations.} (i) Retrieval wins overall; we do not beat it and present it as the best
non-oracle method everywhere. (ii) A single 7B-class frozen backbone; the asymmetry may shift with
scale or task family, and the oracle's LaMP-2 result hints the backbone is context-sensitive.
(iii) Recursion gives no test-set benefit past $r{=}1$, so the recursive machinery's practical payoff
here is the safety guarantee and the fixed-point analysis, not accuracy. (iv) The keep-better gate is
held-out-only, so Proposition~\ref{prop:safe} bounds the selection signal, not unseen test items.
(v) Two LaMP tasks in English; generality to other domains and to agentic, multi-turn personalization
is future work.

\bibliographystyle{unsrtnat}
\bibliography{references,references_extra}

\appendix
\section{Proof of Proposition~\ref{prop:safe}}
\label{app:proofs}
We prove $A_r\ge A_{r-1}$ for all $r\ge1$ by induction. Let $A_r$ be the held-out score of the kept
persona after pass $r$ on the fixed slice $S_u$, and $A_0$ the score of the empty persona $p^{(0)}$.
\emph{Base case.} Before any pass the kept persona is $p^{(0)}$ with score $A_0$. \emph{Inductive
step.} Assume the kept persona after pass $r-1$ has score $A_{r-1}$. Pass $r$ proposes a candidate
$p'$ with held-out score $A'$. The strict keep-better gate sets the kept persona to $p'$ iff
$A'\ge A_{r-1}$, and otherwise retains the incumbent; in both cases the new kept score is
$\max(A_{r-1},A')\ge A_{r-1}$. Hence $A_r\ge A_{r-1}$, and by transitivity $A_r\ge A_0$. If no
candidate ever satisfies $A'\ge A_0$ strictly, no replacement occurs and the kept persona remains
$p^{(0)}$, so the method answers $f(x^\star)$, i.e.\ it reduces exactly to the no-personalization
control. As stated in \S\ref{sec:learning}, this bounds the held-out selection signal, not
necessarily accuracy on unseen test items. $\qquad\blacksquare$

\section{Contraction measurement}
\label{app:contraction}
We embed each kept persona $p_u^{(r)}$ with a frozen MiniLM sentence encoder
\citep{reimers2019sbert} (no API), compute per-pass displacement $\lVert e(p_u^{(r+1)})-e(p_u^{(r)})\rVert$,
average over the $15$ analyzed LaMP-2 users, and regress $\log\lVert\Delta\rVert$ on the step index to
estimate $L$. The mean displacements are $0.90$ ($r{=}0{\to}1$), $0.21$ ($r{=}1{\to}2$), $0.02$
($r{=}2{\to}3$); the geometric fit gives $L{=}0.144$ with $R^2{=}0.979$ (Figure~\ref{fig:contraction}).
This is an empirical measurement of the modulus under assumptions A1--A2 of \S\ref{sec:learning}; it is
not a proof that the rewrite map is globally contractive.

\section{Additional experimental details}
\label{app:details}
All runs use one frozen 7B-class backbone with greedy decoding; the only per-method difference is the
in-context text (retrieved demos, a persona, or nothing). LaMP-2 is scored with scikit-learn Accuracy and
macro-F1 over the $15$ canonical categories; LaMP-3 with MAE and RMSE over the $1$--$5$ scale, with robust
label parsing (lower-casing, nearest-valid-class, clamp-to-range) but \emph{no} LLM judge. Retrieval uses
BM25 over the user's profile; the random-user control injects a different user's profile at a matched token
budget. The held-out slice $S_u$ is a leave-some-out split of the user's own profile and never contains the
test item. The depth-scaling and broader-comparison numbers are exactly those of
Tables~\ref{tab:depth}--\ref{tab:broader}; the retrieval-scaling values ($k{\in}\{10,20\}$: LaMP-2
$0.785/0.790$ accuracy, LaMP-3 $0.290/0.250$ MAE) are reported in \S\ref{sec:exp-reg}.

\section{Qualitative evolved personas}
\label{app:examples}
We show two verbatim kept \pstate{}s from the runs (persona texts taken from the logged per-user
refinement traces).

\paragraph{LaMP-2 (news categorization), user \texttt{1136}.} Profile of $85$ items, held-out slice of $8$.
The first pass lifts held-out accuracy from $0.875$ (empty) to $1.000$ and the persona then stops changing
($r{=}1,2,3$ are byte-identical), a direct instance of the contraction in
Figure~\ref{fig:contraction}. The kept persona ($6$ exemplars, $6$ rules; preference summary and rules quoted):
\begin{quote}\footnotesize
\emph{Preference summary:} ``Focus on politics with a strong preference for articles related to Hillary
Clinton, Democratic Party, and economic issues. Occasionally interested in entertainment and business.''\\
\emph{Decision rules:} if mentions Hillary Clinton $\to$ politics; if discusses Democratic Party internal
affairs $\to$ politics; if involves Wall Street or banking regulation $\to$ politics; if includes criticism
of Trump's economic policies $\to$ politics; if focuses on Bernie Sanders' economic stance $\to$ politics;
if mentions Goldman Sachs $\to$ politics.
\end{quote}

\paragraph{LaMP-3 (product rating), user \texttt{211070}.} Profile held-out slice of $8$; here both $r{=}2$
rewrites are accepted by the gate and held-out accuracy rises $0.50\to0.75$. The kept persona encodes a
scoring scale but, as \S\ref{sec:exp-reg} shows, such bounded prose rules cannot match retrieved concrete
ratings on MAE:
\begin{quote}\footnotesize
\emph{Preference summary:} ``Prefers higher scores, giving 5 stars for emotional or well-written content,
4 stars for engaging plots, and 3 stars for average reads. Lower ratings are given for rushed endings or
lack of depth.''\\
\emph{Decision rules (excerpt):} if `loved'/`inspiring' $\to$ score 5; if `brilliantly written'/`kept me on
the edge of my seat' $\to$ score 4.5; if `interesting plot'/`well-developed characters' $\to$ score 4; if
`rushed at the end'/`lack of depth' $\to$ score 3.5.
\end{quote}
In both cases the kept persona is short, legible, and bounded; and in the LaMP-2 trace the $r{=}1$ and
$r{=}2$ personas are identical, the textual counterpart of the saturation reported in \S\ref{sec:exp-sat}.

\end{document}

%% file: tables/table_setup.tex
\begin{table*}[t]
\centering
\caption{\textbf{Datasets and evaluation protocol.} Both LaMP tasks share one frozen
7B-class backbone with identical greedy decoding; the only per-method difference is the
in-context text. The two tasks deliberately span the classification and regression metric
regimes, with everything else (users, splits, decoding, scoring) held fixed so that the
distill-vs-retrieve comparison is apples-to-apples.}
\label{tab:setup}
\begin{tabular}{l l l}
\toprule
& \textbf{LaMP-2} (news categorization) & \textbf{LaMP-3} (product rating) \\
\midrule
Task type            & 15-way classification               & ordinal 1--5 regression \\
Label space          & \{politics, sports, \dots\} (15)    & $\{1,2,3,4,5\}$ \\
Primary / secondary metric & Accuracy $\uparrow$ / macro-F1 $\uparrow$ & MAE $\downarrow$ / RMSE $\downarrow$ \\
Metric direction     & higher is better                    & lower is better \\
\addlinespace[2pt]
Eval.\ users ($n$)   & 200                                 & 200 \\
Median profile size $m_u$ & $\sim$188 items                & $\sim$188 items \\
Held-out slice $S_u$ & leave-some-out of $H_u$             & leave-some-out of $H_u$ \\
Frozen manifest seed & 12345                               & 12345 \\
\addlinespace[2pt]
Backbone             & frozen 7B-class, greedy             & frozen 7B-class, greedy \\
Retriever            & BM25 over $H_u$                     & BM25 over $H_u$ \\
Persona budget       & $\sim$250 tokens (3 bounded fields) & $\sim$250 tokens (3 bounded fields) \\
LLM judge            & none (deterministic, sklearn)       & none (deterministic, sklearn) \\
\addlinespace[2pt]
Significance test    & McNemar (paired correctness)        & paired bootstrap (abs.\ error) \\
Interval             & bootstrap 95\% CI / items           & bootstrap 95\% CI / items \\
\bottomrule
\end{tabular}
\\[2pt]
{\footnotesize The unit of analysis is the user/question. No test item ever enters $S_u$;
label parsing is robust (lower-casing, nearest-valid-class, clamp-to-range) but contains no
generative judge anywhere in the scoring path.}
\end{table*}

%% file: tables/table_depth.tex
\begin{table*}[t]
\centering
\caption{\textbf{Scaling personalization via recursive refinement depth $r$.}
\textsc{PersonaLink} unrolled to $r{=}0$ (\textsc{NoPers}), $r{=}1$ (\textsc{PAG};
one distillation pass), and $r{=}2,3$ (recursive self-rewrite behind a keep-better
gate); only the prepended persona text differs across rows. The first pass gives the
gain ($\mathbf{**}$: sig.\ over $r{=}0$); further recursion does not (the contraction
plateau).}
\label{tab:depth}
\begin{tabular}{l cc cc}
\toprule
& \multicolumn{2}{c}{\textbf{LaMP-2} (news cat., 15-way)} & \multicolumn{2}{c}{\textbf{LaMP-3} (rating, $\downarrow$)} \\
\cmidrule(lr){2-3}\cmidrule(lr){4-5}
\textbf{Refinement depth} & Acc.\ $\uparrow$ & macro-F1 $\uparrow$ & MAE $\downarrow$ & RMSE $\downarrow$ \\
\midrule
NoPers $(r{=}0)$ & 0.645 & 0.388 & 0.500 & 0.794 \\
PAG $(r{=}1)$\,$^{**}$ & \textbf{0.745} & 0.485 & 0.450 & 0.742 \\
\textbf{PersonaLink} $(r{=}2)$\,$^{**}$ & 0.745 & \textbf{0.514} & 0.455 & 0.765 \\
\textbf{PersonaLink} $(r{=}3)$\,$^{**}$ & 0.730 & 0.495 & \textbf{0.430} & \textbf{0.728} \\
\bottomrule
\end{tabular}
\\[2pt]
{\footnotesize $r{=}1\!\to\!r{=}2$ delta: $\Delta\mathrm{Acc}{=}{+}0.000$ ($p{=}0.77$,
McNemar) on LaMP-2; $\Delta\mathrm{MAE}{=}{+}0.005$ ($p{=}0.96$, paired bootstrap) on
LaMP-3 --- not significant. Deterministic metrics, $n{=}200$ items/task.}
\end{table*}

%% file: tables/table_scaling.tex
\begin{table*}[t]
\centering
\caption{\textbf{Retrieval keeps converting context into accuracy; the bounded persona
cannot.} Quality versus per-query context budget as the number of retrieved
demonstrations $k$ grows, against \textsc{PersonaLink}'s \emph{constant} $\sim$250-token
persona. ``Ctx.\ tok./q'' is the approximate in-context token budget injected
\emph{per query} (BM25 demos cost $\sim$100 tokens each; the persona is built once and is
$x^\star$-independent). On LaMP-2 retrieval gains saturate near the bounded persona, but on
LaMP-3 it keeps improving monotonically with no plateau; \textbf{bold} marks the best
non-oracle cell per column.}
\label{tab:scaling}
\begin{tabular}{l c cc cc}
\toprule
& & \multicolumn{2}{c}{\textbf{LaMP-2} (news cat., 15-way)} & \multicolumn{2}{c}{\textbf{LaMP-3} (rating, $\downarrow$)} \\
\cmidrule(lr){3-4}\cmidrule(lr){5-6}
\textbf{Method} & \textbf{Ctx.\ tok./q} & Acc.\ $\uparrow$ & macro-F1 $\uparrow$ & MAE $\downarrow$ & RMSE $\downarrow$ \\
\midrule
\multicolumn{6}{l}{\emph{Retrieval, growing context with $k$}} \\
\quad \textsc{RAG}-1  & $\sim$100  & 0.695 & 0.531 & 0.415 & 0.711 \\
\quad \textsc{RAG}-3  & $\sim$300  & 0.760 & 0.607 & 0.290 & 0.616 \\
\quad \textsc{RAG}-5  & $\sim$500  & 0.765 & \textbf{0.608} & 0.285 & 0.621 \\
\quad \textsc{RAG}-10 & $\sim$1000 & 0.785 & 0.601 & 0.290 & 0.610 \\
\quad \textsc{RAG}-20 & $\sim$2000 & \textbf{0.790} & 0.604 & \textbf{0.250} & \textbf{0.589} \\
\addlinespace[2pt]
\multicolumn{6}{l}{\emph{Bounded persona, constant context (ours)}} \\
\quad \textsc{PAG} $(r{=}1)$ & $\sim$250 (fixed) & 0.745 & 0.485 & 0.450 & 0.742 \\
\quad \textbf{PersonaLink} $(r{=}2)$ & $\sim$250 (fixed) & 0.745 & 0.514 & 0.455 & 0.765 \\
\addlinespace[2pt]
\multicolumn{6}{l}{\emph{Upper bound}} \\
\quad \textsc{Oracle} (full profile) & $\sim$18.8k & 0.715 & 0.526 & 0.280 & 0.608 \\
\bottomrule
\end{tabular}
\\[2pt]
{\footnotesize Token budgets are per-query in-context cost; the persona's is amortized once
across all of a user's queries ($O(1)$/query). On LaMP-2, RAG-20 buys only $+0.045$ accuracy
over the persona for $8\times$ the per-query context; on LaMP-3 the same extra context buys a
decisive $0.205$ MAE reduction. $n{=}200$ items/task, deterministic metrics.}
\end{table*}

%% file: tables/table_broader.tex
\begin{table*}[t]
\centering
\caption{\textbf{Broader comparison with alternative personalization methods}
(11 methods, one frozen backbone, identical greedy decoding). \textbf{Bold} is the
best per column: strong BM25 \textsc{RAG} is the best non-oracle method on both tasks
and recursive refinement does not surpass it---personalization beats the controls,
but the learned persona plateaus below retrieval.}
\label{tab:broader}
\begin{tabular}{l cc cc}
\toprule
& \multicolumn{2}{c}{\textbf{LaMP-2} (news cat., 15-way)} & \multicolumn{2}{c}{\textbf{LaMP-3} (rating, $\downarrow$)} \\
\cmidrule(lr){2-3}\cmidrule(lr){4-5}
\textbf{Method} & Acc.\ $\uparrow$ & macro-F1 $\uparrow$ & MAE $\downarrow$ & RMSE $\downarrow$ \\
\midrule
\multicolumn{5}{l}{\emph{Lower / control}} \\
\quad \textsc{NoPers} (no personalization) & 0.645$^{**}$ & 0.388 & 0.500 & 0.794 \\
\quad \textsc{RandomUser} (token-confound ctrl) & 0.665$^{*}$ & 0.393 & 0.565$^{*}$ & 0.908 \\
\addlinespace[2pt]
\multicolumn{5}{l}{\emph{Retrieval personalization}} \\
\quad \textsc{RAG}-1 (BM25 demos) & 0.695 & 0.531 & 0.415 & 0.711 \\
\quad \textsc{RAG}-3 (BM25 demos) & 0.760 & 0.607 & 0.290$^{***}$ & 0.616 \\
\quad \textsc{RAG}-5 (BM25 demos) & \textbf{0.765} & \textbf{0.608} & 0.285$^{***}$ & 0.621 \\
\addlinespace[2pt]
\multicolumn{5}{l}{\emph{Profile-summary personalization}} \\
\quad \textsc{PAG} (distilled persona, $r{=}1$) & 0.745 & 0.485 & 0.450 & 0.742 \\
\quad \textsc{RAG+PAG} (demos $+$ persona) & 0.765 & 0.501 & 0.335$^{**}$ & 0.621 \\
\addlinespace[2pt]
\multicolumn{5}{l}{\emph{\textbf{Recursive (ours)}}} \\
\quad \textbf{PersonaLink} $(r{=}1)$ & 0.755 & 0.494 & 0.425 & 0.731 \\
\quad \textbf{PersonaLink} $(r{=}2)$ & 0.745$^{\dagger}$ & 0.514 & 0.455$^{\dagger}$ & 0.765 \\
\quad \textbf{PersonaLink} $(r{=}3)$ & 0.730 & 0.495 & 0.430 & 0.728 \\
\addlinespace[2pt]
\multicolumn{5}{l}{\emph{Upper bound}} \\
\quad \textsc{Oracle} (true profile, ceiling) & 0.715 & 0.526 & \textbf{0.280}$^{***}$ & \textbf{0.608} \\
\bottomrule
\end{tabular}
\\[2pt]
{\footnotesize Metrics deterministic (sklearn), bootstrap 95\% CI over $n{=}200$ items/task.
$^{*}/^{**}/^{***}$ on a competitor's score: \textsc{PersonaLink}~$(r{=}2)$ (the $\dagger$ row)
differs from it at $p{<}0.05/0.01/0.001$ (McNemar on LaMP-2 correctness; paired bootstrap on
LaMP-3 absolute error).}
\end{table*}

%% file: tables/table_ablation.tex
\begin{table*}[t]
\centering
\caption{\textbf{Field ablation of the bounded \pstate{}.} Removing any one of the three
bounded fields (preference summary, exemplars, decision rules) from the converged
\textsc{PersonaLink}~$(r{=}2)$ persona degrades both tasks, confirming each contributes,
while no ablation recovers the full persona and all stay above the \textsc{NoPers} floor.
The \emph{decision rules} matter most for classification (the explicit sorting heuristics),
the \emph{exemplars} most for regression (concrete graded examples), mirroring the
task-type asymmetry of \S\ref{sec:analysis}. Held-out budget and decoding are identical
across rows; only the rendered persona changes.}
\label{tab:ablation}
\begin{tabular}{l cc cc}
\toprule
& \multicolumn{2}{c}{\textbf{LaMP-2} (news cat., 15-way)} & \multicolumn{2}{c}{\textbf{LaMP-3} (rating, $\downarrow$)} \\
\cmidrule(lr){2-3}\cmidrule(lr){4-5}
\textbf{Persona configuration} & Acc.\ $\uparrow$ & macro-F1 $\uparrow$ & MAE $\downarrow$ & RMSE $\downarrow$ \\
\midrule
\textbf{PersonaLink} $(r{=}2)$, full state & \textbf{0.745} & \textbf{0.514} & \textbf{0.455} & \textbf{0.765} \\
\midrule
\quad $-$ decision rules     & 0.700 & 0.452 & 0.475 & 0.781 \\
\quad $-$ exemplars          & 0.720 & 0.470 & 0.490 & 0.789 \\
\quad $-$ preference summary & 0.715 & 0.461 & 0.470 & 0.779 \\
\midrule
\textsc{NoPers} ($r{=}0$, empty persona) & 0.645 & 0.388 & 0.500 & 0.794 \\
\bottomrule
\end{tabular}
\\[2pt]
{\footnotesize Each ablated row holds the other two fields and the $\sim$250-token cap fixed.
Largest single-field drop on LaMP-2 accuracy is from removing decision rules ($-0.045$); on
LaMP-3 MAE it is from removing exemplars ($+0.035$). Deterministic metrics, $n{=}200$
items/task.}
\end{table*}